\documentclass[letterpaper, 10 pt, conference]{ieeeconf}  %

\IEEEoverridecommandlockouts                              %

\usepackage{hyperref}
\usepackage{url}
\usepackage{graphicx}
\usepackage[ruled, vlined,linesnumbered]{algorithm2e}
\usepackage{mathtools}
\usepackage{dsfont}
\usepackage{booktabs}
\usepackage{subfig}
\usepackage{listings}
\usepackage{amssymb}
\usepackage{xspace}
\usepackage{wrapfig}
\usepackage{cuted} %
\usepackage{cite}
\usepackage{caption}
\usepackage{multirow}
\usepackage{censor}
\usepackage{placeins}

\usepackage{xcolor} %
 \definecolor{mygreen}{rgb}{0,0.6,0}
 \definecolor{mygray}{rgb}{0.5,0.5,0.5}
 \definecolor{mymauve}{rgb}{0.58,0,0.82}
 \definecolor{backcolour}{rgb}{0.95,0.95,0.92}

\usepackage{colortbl}

\newtheorem{theorem}{Theorem}[section]

\newtheorem{definition}[theorem]{Definition}
\newtheorem{lemma}[theorem]{Lemma}

\let\oldthebibliography\thebibliography
\let\endoldthebibliography\endthebibliography
\renewenvironment{thebibliography}[1]{%
  \oldthebibliography{#1}\interlinepenalty=10000
}{\endoldthebibliography}

\newcommand{\dadiff}{\textsc{DADiff}\xspace}

\title{\LARGE \bf
\dadiff: Diffusion-Driven Cross-Domain Policy Adaptation for Reinforcement Learning
}

\author{Hanyang Chen$^{1}$, Anirudh Satheesh$^{2}$, Longchao Da$^{1}$ and Hua Wei$^{1}$ %
\thanks{$^{1}$Hanyang Chen, Longchao Da, and Hua Wei are with Arizona State University, {\tt\small \{hchen478, longchao, hua.wei\}@asu.edu}}%
\thanks{$^{2}$Anirudh Satheesh is with the University of Maryland, College Park, {\tt\small anirudhs@terpmail.umd.edu}}%
}

\begin{document}

\maketitle
\thispagestyle{empty}
\pagestyle{empty}

\begin{abstract}

Transferring policies across domains poses a vital challenge in reinforcement learning, due to the dynamics mismatch between the source and target domains. In this paper, we consider the setting of online dynamics adaptation, where policies are trained in the source domain with sufficient data, while only limited interactions with the target domain are allowed. There are a few existing works that address the dynamics mismatch by employing domain classifiers, value-guided data filtering, or representation learning. Instead, we study the domain adaptation problem from a generative modeling perspective. Specifically, we introduce \dadiff, a diffusion-based framework that leverages the discrepancy between source and target domain generative trajectories in the generation process of the next state to estimate the dynamics mismatch. Both reward modification and data selection variants are developed to adapt the policy to the target domain. We also provide a theoretical analysis to show that the performance difference of a given policy between the two domains is bounded by the generative trajectory deviation. More discussions on the applicability of the variants and the connection between our theoretical analysis and the prior work are further provided. We conduct extensive experiments in environments with various shifts to validate the effectiveness of our method. The results demonstrate that our method provides superior performance compared to existing approaches, effectively addressing the dynamics mismatch. We provide the code of our method at \url{https://github.com/hanyang-chen/DADiff-release}.

\end{abstract}

\section{Introduction}
\label{sec:introduction}

Reinforcement learning (RL) has shown strong potential in complex decision-making tasks, but training directly in the real-world environment (\textit{target domain}) is often restricted by safety, cost, and limited interaction budgets. An alternative strategy is to train policies in a surrogate environment (\textit{source domain}), such as a simulator, and then transfer them to the target domain. But due to the dynamics mismatch between the source and target domains, directly transferring the policy often leads to performance degradation, which is a critical challenge in the sim-to-real problem \cite{zhao2020sim, da2025survey}. One solution to this transfer problem is known as \textit{online dynamics adaptation} \cite{xu2023crossdomain, lyu2024cross}, where policies are trained with abundant source-domain data and only limited interactions in the target domain. In this setting, the state space, action space, and reward function remain consistent across domains, while the transition dynamics differ. Compared with solutions such as domain randomization \cite{peng2018sim, mehta2020active, curtis2025flowbased} or simulator calibration \cite{chebotar2019closing}, online dynamics adaptation does not require access to high-fidelity simulators or prior knowledge of target dynamics, and can therefore be applied in situations where such information is unavailable.

Existing online dynamics adaptation methods, including classifier-based approaches \cite{eysenbach2021offdynamics}, value-guided filtering \cite{xu2023crossdomain}, and representation learning \cite{lyu2024crossdomain}, capture dynamics discrepancy from different perspectives: classifiers provide coarse distinctions between domains, value-guided methods depend on the modeling of forward predictions, and representation learning relies on assumptions of invariant latent structures across domains. When the domains are complex or stochastic, a key challenge that remains is to develop an approach capable of capturing dynamics discrepancy in a more fine-grained and distributional manner.

The generative modeling perspective provides a potential direction. Generative models, such as diffusion models \cite{sohl2015deep, ho2020denoising}, have demonstrated strong capability in representing complex distributions. When state transitions are viewed as a conditional generative process, the mismatch between source and target domains can be interpreted as a discrepancy between their respective generative processes. Specifically, the multi-step sampling procedure in diffusion models and flow matching methods produces several latent states, which construct a generative trajectory, serving as structured signals of source–target dynamics deviation. These latent states allow the discrepancy to be captured not only at the next-state level but also along the entire trajectory. Intuitively, if the source and target domains follow different dynamics, their trajectories will diverge at multiple steps, a phenomenon we term \emph{generative trajectory deviation}. This notion provides a fine-grained view of dynamics discrepancy by revealing how divergence accumulates along the trajectory, rather than relying solely on local or aggregated comparisons. Our theoretical analysis further connects trajectory deviation to performance guarantees, providing motivation for algorithmic design.

Building on this perspective, we introduce \dadiff, a diffusion-based framework for online dynamics adaptation. \dadiff leverages latent states in diffusion models to measure generative trajectory deviation between source and target domains, and exploits this deviation in two complementary ways: (i) \dadiff-modify, which adjusts source-domain rewards with deviation-based penalties, and (ii) \dadiff-select, which filters source-domain data based on deviation before value function updates. We further discuss the applicability of these variants to different tasks, highlight the advantages of our method compared to prior work, and establish a connection between our analysis and the theoretical guarantee of prior work. Empirical results in environments with various shifts show the superior performance of our method compared to existing algorithms.

\section{Related Works}
\label{sec:related-works}

\paragraph{Domain Adaptation in RL}
\label{sec:domain-adaptation-rl}
Generalizing RL policies to diverse environments is critical for real-world deployment, where transition dynamics \cite{eysenbach2021offdynamics, xue2023state}, state or action spaces \cite{ge2023policy, pan2025crossdomain} may be different. To address domain adaptation, prior work falls under three categories: (i) domain randomization that randomizes transition dynamics to expose agents to many environment configurations \cite{slaoui2019robust, jiang2023variance}, (ii) meta-learning to few-shot adapt to many environments \cite{nagabandi2018learning, wu2022zero}, and (iii) expert demonstrations of target environments through imitation learning \cite{raychaudhuri2021cross, fickinger2022crossdomain}. However, these approaches are either computationally expensive or require hard-to-obtain demonstrations. With only limited target-domain data, some works perform reward modifications to transition to the target domain by using transition classifiers \cite{eysenbach2021offdynamics, guo2024off} or reward augmentations \cite{van2024policy, lyu2024cross}. Data selection methods \cite{xu2023crossdomain, wen2024contrastive} have also been used to filter out part of the source-domain transitions and train policies on both source and target domain data. When the domains are complex or stochastic, a key challenge that remains is to develop an approach capable of capturing the dynamics discrepancy. Our method explores this challenge from a generative modeling perspective by measuring the generative trajectory deviation between the source and target domains.

\paragraph{Diffusion Models in RL}
\label{sec:diffusion-models-rl}
Diffusion models \cite{sohl2015deep, ho2020denoising} have been extensively used for generating effective decision-making policies in several domains, such as reinforcement learning \cite{kang2023efficient} and robotics \cite{chi2025diffusion}. Specifically, they are widely leveraged to synthesize data for offline RL \cite{lu2023synthetic}, facilitate planning and action generation in multi-task scenarios \cite{he2023diffusion}, and enhance the representational capacity of learned RL policies \cite{wang2024diffusion}. In addition, diffusion models have also been extended to the multi-agent settings \cite{zhu2024madiff}. In the field of domain adaptation, they are utilized to augment the target-domain data in order to boost the performance of offline RL policies \cite{van2025dmc}. However, the introduction of synthesizers may lead to extra computational costs, and the quality of synthesized data is hard to guarantee. In contrast, we choose to directly estimate the dynamics discrepancy by multiple latent states from diffusion models instead of generating more synthetic data.

\section{Preliminaries}
\label{sec:preliminaries}

\paragraph{Online Dynamics Adaptation}
\label{sec:off-dynamics-rl}

We consider two Markov Decision Processes (MDPs), denoted as $\mathcal{M}_\mathrm{src} = (\mathcal{S}, \mathcal{A}, P_\mathrm{src}, r, \gamma)$ and $\mathcal{M}_\mathrm{tar} = (\mathcal{S}, \mathcal{A}, P_\mathrm{tar}, r, \gamma)$ for the source domain and target domain, respectively. The state space $\mathcal{S}$, action space $\mathcal{A}$, reward function $r:\mathcal{S} \times \mathcal{A} \rightarrow \mathbb{R}$ and discount factor $\gamma \in [0, 1]$ are consistent across both domains, while the transition dynamics
$P_\mathrm{src}$ and $P_\mathrm{tar}$ differ. The goal of online dynamics adaptation is to learn a policy $\pi$ that achieves high performance in the target domain $\mathcal{M}_\mathrm{tar}$, utilizing sufficient data from the source domain and only limited interactions from the target domain.
In addition, we specify a domain $\mathcal{M}$ and define the probability that a policy $\pi$ encounters a state $s$ at time step $t$ as $P_{\mathcal{M},t}^{\pi}(s)$.
Therefore, the normalized probability that a policy $\pi$ visits a state-action pair $(s,a)$ in the domain $\mathcal{M}$ can be represented as $\rho_{\mathcal{M}}^{\pi}(s,a) \coloneqq (1-\gamma) \sum_{t=0}^{\infty} \gamma^t P_{\mathcal{M},t}^{\pi}(s) \pi(a|s)$. The expected return of a policy $\pi$ in $\mathcal{M}$ is defined as $\eta_{\mathcal{M}}(\pi) = \mathbb{E}_{(s,a) \sim \rho_{\mathcal{M}}^{\pi}}[r(s,a)]$. We assume the reward are bounded by $|r(s,a)| \le r_\mathrm{max}, \forall s \in \mathcal{S}, a \in \mathcal{A}$.

\paragraph{Diffusion Models}
\label{sec:diffusion-models}

Diffusion models \cite{sohl2015deep, ho2020denoising} are a family of generative models that learn to generate samples from a target distribution. We mainly focus on the denoising diffusion probabilistic model (DDPM) \cite{ho2020denoising} in this paper.
DDPM consists of a forward process and a reverse process. The forward process is regarded as a Markov chain that gradually adds noise to data, transforming a clean data point $x_0$ into Gaussian noise, which is formulated as follows,
\begin{equation}
    x_k = \sqrt{1 - \beta_k} x_{k-1} + \sqrt{\beta_k} \epsilon, \quad \epsilon \sim \mathcal{N}(0, I),
\end{equation}
where $x_k$ is the noisy data at diffusion timestep $k$, $\beta_k$ is the noise schedule, and $\epsilon$ is Gaussian noise.
To simplify the forward process, we can directly sample the noisy data at diffusion timestep $k$ as follows,
\begin{equation}
    x_k = \sqrt{\bar{\alpha}_k} x_0 + \sqrt{1 - \bar{\alpha}_k} \epsilon, \quad \epsilon \sim \mathcal{N}(0, I),
\end{equation}
where $\alpha_k = 1 - \beta_k$ and $\bar{\alpha}_k = \prod_{i=1}^k \alpha_i$. 
The reverse process learns to denoise the noisy data step by step, which is formulated as follows,
\begin{equation}
    \resizebox{\linewidth}{!}{$
    \begin{aligned}
    x_{k-1} &= \frac{1}{\sqrt{\alpha_k}} \left(x_k - \frac{\beta_k}{\sqrt{1 - \bar{\alpha}_k}} \epsilon_\theta(x_k, k)\right) + \sqrt{\frac{1 - \bar{\alpha}_{k-1}}{1 - \bar{\alpha}_k}\beta_k}\,\epsilon,\quad \epsilon \sim \mathcal{N}(0, I),
    \end{aligned}
    $}
\end{equation}
where $\epsilon_\theta(x_k, k)$ is a noise model that estimates the noise from the noisy data point $x_k$. The noisy data points $\{x_k\}_{k=0}^K$ form a generative trajectory from the initial noisy data $x_K$ to the clean data $x_0$.
The training objective of the noise model is formulated as follows,
\begin{equation}
    \mathcal{L}_\mathrm{diff} = \mathbb{E}_{x_0, \epsilon, k} \left[ ||\epsilon - \epsilon_\theta(\sqrt{\bar{\alpha}_k} x_0 + \sqrt{1 - \bar{\alpha}_k} \epsilon, k)||^2 \right].
    \label{eq:training-objective-diffusion-models}
\end{equation}

\section{Methodology}
\label{sec:method}

\begin{figure}
    \centering
    \vspace{0.15cm}
      \setlength{\belowcaptionskip}{-0.4cm}
      \includegraphics[width=1.0\linewidth]{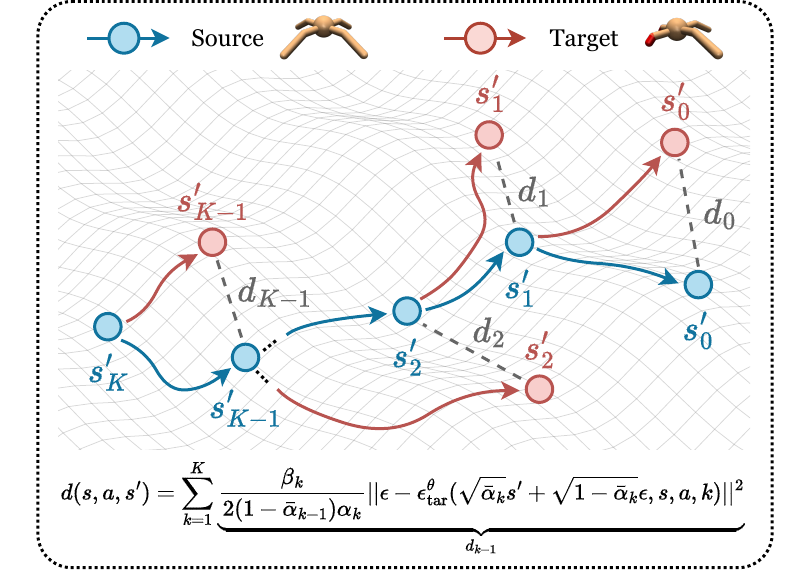}
      \caption{
        Illustration of \dadiff. This figure visualizes the generative trajectories in the source and target domains. The deviation $d(s,a,s')$ is measured by the discrepancy $d_k$ of each latent state $s'_k$ in the source and target domain generative trajectories. %
      }
      \label{fig:model-architecture}
\end{figure}

In this section, we first introduce a theoretical analysis to demonstrate the connection between the dynamics mismatch and the generative trajectory mismatch. Then, we present our diffusion-based method, \dadiff, which measures the generative trajectory deviation from the perspective of diffusion models and adapts the learned policy to the target domain.
The overview of our method is shown in Figure~\ref{fig:model-architecture}.

\subsection{Theoretical Analysis}
\label{sec:theoretical-analysis}

Before introducing the theoretical analysis, we first provide the definition of a generative trajectory, which is crucial for the analysis. For clarity, we denote the next state $s'$ as $s'_0$.

\begin{definition}
{
    (Generative trajectory.)
    Specify a domain $\mathcal{M}$ with transition dynamics $P_\mathcal{M}(s'_0 | s, a)$. There is a generative trajectory for the next state $s'_0$ consisting of $K$ auxiliary variables $\{s'_k\}_{k=1}^K$, referred to as latent states. These latent states form a Markov chain from the initial latent state $s'_K$ to the next state $s'_0$ conditioned on the state–action pair $(s,a)$.
    \label{definition:bridge-states}
}
\end{definition}

\textbf{\textit{Remark.}}
The Markov-chain definition enables the transition dynamics to be decomposed into multiple conditional probabilities, \textit{i.e.}, $P_\mathcal{M}(s'_0 | s, a) = \int P_\mathcal{M}(s'_K | s, a) \prod_{k=1}^{K} P_\mathcal{M}(s'_{k-1} | s'_k, s, a) ds'_{1:K}$. In this way, the next state $s'_0$ can be viewed as being generated step by step with latent states, forming a generative trajectory. The discrepancy of such generative trajectories across domains provides a natural estimation of the dynamics discrepancy.

We construct generative trajectories in both source and target domains, starting from the same initial latent state $s'_K$, and derive Theorem~\ref{theorem:performance-bound-latent-state-discrepancy} to establish the connection between the dynamics mismatch and the generative trajectory mismatch. The detailed proof is provided in Appendix~\ref{sec:proof-theorem}.

\begin{theorem}{
    (Performance bound controlled by generative trajectory discrepancy.)
    Denote $\mathcal{M}_\mathrm{src}$ and $\mathcal{M}_\mathrm{tar}$ as the source and target domains with different dynamics, respectively. The performance difference of any policy $\pi$ evaluated in $\mathcal{M}_\mathrm{src}$ and $\mathcal{M}_\mathrm{tar}$ can be bounded as below,
        \begin{equation}
        \resizebox{\linewidth}{!}{$
            \begin{aligned}
            & \eta_{\mathcal{M}_\mathrm{src}}(\pi) - \eta_{\mathcal{M}_\mathrm{tar}}(\pi)
            \le \\
            & \frac{\sqrt{2}\gamma r_{\mathrm{max}}}{(1 - \gamma)^2}
            \underbrace{\mathbb{E}_{\rho_\mathrm{src}^\pi}\!\left[
            \sqrt{\mathbb{E}_{P_\mathrm{src}}\!\left[
            D_\mathrm{KL}\!\left(P_\mathrm{src}(s'_K \mid s, a)\,\|\,P_\mathrm{tar}(s'_K \mid s, a)\right)
            \right]}
            \right]}_{(a):\ \mathrm{initial\ latent\ state\ deviation}} + \\
            & \frac{\sqrt{2}\gamma r_{\mathrm{max}}}{(1 - \gamma)^2}
            \underbrace{\mathbb{E}_{\rho_\mathrm{src}^\pi}\!\left[
            \sqrt{\mathbb{E}_{P_\mathrm{src}}\!\left[
            \sum_{k=1}^{K}
            D_\mathrm{KL}\!\left(P_\mathrm{src}(s'_{k-1} \mid s'_k, s, a)\,\|\,P_\mathrm{tar}(s'_{k-1} \mid s'_k, s, a)\right)
            \right]}
            \right]}_{(b):\ \mathrm{latent\ state\ transition\ mismatch}}.
            \end{aligned}
            $}
            \label{eq:performance-bound-latent-state-discrepancy}
        \end{equation}
    
    \label{theorem:performance-bound-latent-state-discrepancy}
}
\end{theorem}

\textbf{\textit{Remark.}}
This bound indicates that the performance difference of a policy $\pi$ between the source and target domains is controlled by the initial latent state deviation term (a) and the latent state transition mismatch term (b). Since the generative trajectories in both the source and target domains share the same initial latent state $s'_K$, term (a) vanishes, leaving term (b) as the sole determinant of the performance difference. In other words, as long as the generative trajectories are similar in the source and target domains, the performance difference is small, and vice versa. We note that PAR \cite{lyu2024crossdomain} can be considered as a special case of Theorem~\ref{theorem:performance-bound-latent-state-discrepancy} when $K=1$.
A discussion on the connection between our analysis and the theoretical guarantee of PAR is provided in Section~\ref{sec:connection-dadiff-par}.

\subsection{Domain Adaptation with Diffusion}
\label{sec:domain-adaptation}

Theorem~\ref{theorem:performance-bound-latent-state-discrepancy} provides a theoretical guarantee linking the performance difference of a policy $\pi$ to the generative trajectory, thereby motivating a careful design of latent states in the trajectory. In this section, we adopt the formulation of DDPM to better characterize the dynamics discrepancy. 

We first redeclare the reverse process of DDPM in a reparameterized form to describe the latent state transition in domain $\mathcal{M}$ as follows,
\begin{equation}
    \resizebox{\linewidth}{!}{$
    \begin{aligned}
    s'_{k-1} = \frac{1}{\sqrt{\alpha_k}} \left(s'_k - \frac{\beta_k}{\sqrt{1 - \bar{\alpha}_k}} \epsilon_\mathcal{M}(s'_k, s, a, k)\right) + \sqrt{\frac{1 - \bar{\alpha}_{k-1}}{1 - \bar{\alpha}_k}\beta_k}\,\epsilon,\quad \epsilon \sim \mathcal{N}(0, I),
    \end{aligned}
    $}
    \label{eq:latent-state-transition-ddpm}
\end{equation}
where $\epsilon_\mathcal{M}(s'_k, s, a, k)$ is the noise from the latent state $s'_k$ in domain $\mathcal{M}$. It indicates that the latent state transition follows a Gaussian distribution, \textit{i.e.}, 
\begin{equation}
    \resizebox{\linewidth}{!}{$
    \begin{aligned}
    P_\mathcal{M}(s'_{k-1} & \mid s'_k, s, a) 
    &\sim \mathcal{N}\Bigg(
    \frac{1}{\sqrt{\alpha_k}} \Big(s'_k - \frac{\beta_k}{\sqrt{1 - \bar{\alpha}_k}} \epsilon_\mathcal{M}(s'_k, s, a, k)\Big), \frac{1 - \bar{\alpha}_{k-1}}{1 - \bar{\alpha}_k}\beta_k I
    \Bigg).
    \end{aligned}
    $}
    \label{eq:latent-state-transition-ddpm-distribution}
\end{equation}
According to Theorem~\ref{theorem:performance-bound-latent-state-discrepancy}, the performance difference of a policy $\pi$ across domains is determined by the latent state transition mismatch term (b). Therefore, we can estimate the generative trajectory deviation $d(s,a,s')$ with the defined distribution of latent state transition in Equation~\ref{eq:latent-state-transition-ddpm-distribution} as follows,
\begin{equation}
    \resizebox{0.95\linewidth}{!}{$
    \begin{aligned}
        d(s,a,s') &= \sum_{k=1}^{K} D_\mathrm{KL}(P_\mathrm{src}(s'_{k-1} | s'_k, s, a) || P_\mathrm{tar}(s'_{k-1} | s'_k, s, a)) \\
        &= \sum_{k=1}^{K} \frac{\beta_k}{2(1 - \bar{\alpha}_{k-1}) \alpha_k} \left\Vert \epsilon_\mathrm{src}(s'_k, s, a, k) - \epsilon_\mathrm{tar}(s'_k, s, a, k) \right\Vert^2.
    \end{aligned}
    $}
    \label{eq:generative-trajectory-deviation}
\end{equation}
We derive this equation by computing the KL divergence between two Gaussian distributions.
Notably, as the state transition tuple $(s, a, s')$ comes from the source domain, the noise $\epsilon_\mathrm{src}(s'_k, s, a, k)$ estimated in the reverse process must be consistent with the noise used in the forward process to generate the latent state $s'_k$, which indicates $\epsilon_\mathrm{src}(s'_k, s, a, k) = \epsilon$ with $\epsilon \sim \mathcal{N}(0, I)$. Besides, we introduce a noise model $\epsilon^\theta_\mathrm{tar}(s'_k, s, a, k)$, trained with target-domain data, to estimate the noise in the target domain. The training objective is formulated as follows,
\begin{equation}
    \resizebox{\linewidth}{!}{$
    \mathcal{L}_\mathrm{noise} = \mathbb{E}_{(s,a,s')\sim\mathcal{D}_\mathrm{tar}, \epsilon, k} \left[ \left\Vert \epsilon - \epsilon^\theta_\mathrm{tar}(\sqrt{\bar{\alpha}_k} s'_0 + \sqrt{1 - \bar{\alpha}_k} \epsilon, s, a, k) \right\Vert^2 \right].
    $}
    \label{eq:training-objective-diffusion-models-rewritten}
\end{equation}
This objective mirrors the standard DDPM training loss, but conditions on $(s,a)$ to capture dynamics in the target domain.
For the latent state $s'_k$ in Equation~\ref{eq:generative-trajectory-deviation}, there are two ways to obtain it: (i) by iteratively applying the reverse process in Equation~\ref{eq:latent-state-transition-ddpm}, and (ii) by sampling directly from the forward process of DDPM, \textit{i.e.}, $s'_k = \sqrt{\bar{\alpha}_k} s'_0 + \sqrt{1 - \bar{\alpha}_k} \epsilon$ with $\epsilon \sim \mathcal{N}(0, I)$. Specifically, the first way requires sequential sampling across all steps to generate the entire generative trajectory, which is computationally expensive. In contrast, the second way can produce all latent states in parallel, yielding a much more efficient implementation. Therefore, we choose to obtain the latent state $s'_k$ via the forward process in our method. 
Finally, the deviation $d(s,a,s')$ can be practically estimated as follows,
\begin{equation}
    \resizebox{\linewidth}{!}{$
    \begin{aligned}
    d(s,a,s') = \sum_{k=1}^{K} \frac{\beta_k}{2(1 - \bar{\alpha}_{k-1}) \alpha_k} \left\Vert \epsilon - \epsilon^\theta_\mathrm{tar}(\sqrt{\bar{\alpha}_k} s'_0 + \sqrt{1 - \bar{\alpha}_k} \epsilon, s, a, k) \right\Vert^2,\\
    \epsilon \sim \mathcal{N}(0, I).
    \end{aligned}
    $}
    \label{eq:estimated-generative-trajectory-deviation}
\end{equation}
We further introduce two variants based on SAC \cite{haarnoja2018soft} to utilize the deviation $d(s,a,s')$, including reward modification and data selection, since we find that baselines adopting these two techniques exhibit complementary advantages in different tasks, which is shown in Section~\ref{sec:adaptation-performance}. We analyze the possible reason for this phenomenon from the reward distribution aspect in Section~\ref{sec:reward-distribution-analysis}. The details of \dadiff variants are provided as follows.

\paragraph{Reward modification.}
We refer to this variant as \dadiff-modify. It adopts the deviation $d(s,a,s')$ as a reward penalty to modify the reward function in the source domain, \textit{i.e.},
\begin{equation}
    r_\mathrm{mod}(s,a,s') = r(s,a,s') - \lambda d(s,a,s'),
    \label{eq:reward-modification}
\end{equation}
where $\lambda$ is a penalty coefficient to balance the original reward and the penalty. The objective function for training the value function gives,
\begin{equation}
    \mathcal{L}_\mathrm{critic} = \mathbb{E}_{(s, a, r_\text{mod}, s') \sim \mathcal{D}_\mathrm{src} \cup \mathcal{D}_\mathrm{tar}} \left[ (Q_\phi - \mathcal{T}Q_\phi)^2 \right],
    \label{eq:critic-objective-modification}
\end{equation}
where $\mathcal{D}_\mathrm{tar}$ and $\mathcal{D}_\mathrm{src}$ are the datasets from the target and source domains, respectively, $Q_\phi$ is the value function, and $\mathcal{T}$ is the Bellman operator.

\paragraph{Data selection.}
We refer to this variant as \dadiff-select. We select fixed percentage data with the lowest deviation $d(s,a,s')$ from a batch of source domain data. The selected data is then used to update the value function. We formulate the objective function of the value function as follows,
\begin{equation}
    \begin{aligned}
    \mathcal{L}_\mathrm{critic} =\ & \mathbb{E}_{(s, a, r, s') \sim \mathcal{D}_\mathrm{tar}} \left[ (Q_\phi - \mathcal{T}Q_\phi)^2 \right] + \\
    & \mathbb{E}_{(s, a, r, s') \sim \mathcal{D}_\mathrm{src}} \left[ \omega(s,a,s') (Q_\phi - \mathcal{T}Q_\phi)^2 \right],
    \end{aligned}
    \label{eq:critic-objective-selection}
\end{equation}
where $\omega(s,a,s') = \mathds{1}(d(s,a,s') < d_{\xi\%})$, $\mathds{1}$ is the indicator function, and $d_{\xi\%}$ denotes the lowest $\xi$-quantile deviation in the batch.

For both variants, the objective function of the policy $\pi$ is formulated as:
\begin{equation}
    \resizebox{\linewidth}{!}{$
    \mathcal{L}_\mathrm{actor} = \mathbb{E}_{(s, a, r, s') \sim \mathcal{D}_\mathrm{src} \cup \mathcal{D}_\mathrm{tar}} \left[ -\min_{i=1,2} Q_{\phi_i}(s,a) + \tau \log \pi(a|s) \right],
    $}
    \label{eq:actor-objective}
\end{equation}
where $\tau$ is the entropy temperature coefficient, and $i$ denotes the value function index. 
We provide the pseudocode of \dadiff in Algorithm \ref{algo:dadiff}.

\begin{algorithm}
\footnotesize
\SetInd{0.25em}{1em}
\caption{Domain Adaptation with \dadiff}
\label{algo:dadiff}
\DontPrintSemicolon
\SetAlgoNlRelativeSize{-1}
\SetKwInput{KwInput}{Input}
\SetKwInput{KwInit}{Initialization}

\KwInput{Source domain $\mathcal{M}_\mathrm{src}$, target domain $\mathcal{M}_\mathrm{tar}$, and target domain interaction frequency $F$}
\KwInit{Policy $\pi$, value function $\{Q_{\phi_i}\}_{i=1,2}$, target value function $\{Q_{\phi'_i}\}_{i=1,2}$, noise model $\epsilon^\theta_\mathrm{tar}$, replay buffers $\{\mathcal{D}_\mathrm{src}, \mathcal{D}_\mathrm{tar}\}$, penalty coefficient $\lambda$, data selection ratio $\xi$, batch size $N$}

\For{$j = 1, 2, \dots$}{
    Collect $(s_\mathrm{src}, a_\mathrm{src}, r_\mathrm{src}, s'_\mathrm{src})$ from $\mathcal{M}_\mathrm{src}$, store in $\mathcal{D}_\mathrm{src}$\;
    \If{$j \bmod F = 0$}{
        Collect $(s_\mathrm{tar}, a_\mathrm{tar}, r_\mathrm{tar}, s'_\mathrm{tar})$ from $\mathcal{M}_\mathrm{tar}$, store in $\mathcal{D}_\mathrm{tar}$\;
    }
    Sample $N$ transitions from $\mathcal{D}_\mathrm{tar}$, train model $\epsilon^\theta_\mathrm{tar}$ via Eq.~\ref{eq:training-objective-diffusion-models-rewritten}\;
    Sample $N$ transitions from $\mathcal{D}_\mathrm{src}$, compute $d(s_\mathrm{src}, a_\mathrm{src}, s'_\mathrm{src})$ via Eq.~\ref{eq:estimated-generative-trajectory-deviation}\;
    \If{using reward modification}{
        Modify source domain rewards via Eq.~\ref{eq:reward-modification}\;
        Update value functions $Q_{\phi_i}$ by minimizing Eq.~\ref{eq:critic-objective-modification}\;
    }\ElseIf{using data selection}{
        Select $\xi$-quantile data from $\mathcal{D}_\mathrm{src}$ by $d(s_\mathrm{src}, a_\mathrm{src}, s'_\mathrm{src})$\;
        Update value functions $Q_{\phi_i}$ by minimizing Eq.~\ref{eq:critic-objective-selection}\;
    }
    Update actor $\pi$ by minimizing Eq.~\ref{eq:actor-objective}\;
    Update target value functions $Q_{\phi'_i}$\;
}
\end{algorithm}

\section{Experiments}
\label{sec:experiments}

\subsection{Experimental Setup}
\label{sec:experimental-setup}

We conduct experiments in four environments (\textit{ant}, \textit{hopper}, \textit{halfcheetah}, \textit{walker}) from Gym MuJoCo \cite{todorov2012mujoco, brockman2016openai}. The source domain is set as the original environment, while the target domain is set as the environment with shifts in kinematics, morphology, friction, or gravity. Kinematic shifts restrict joint rotation ranges, morphology shifts reduce limb sizes, friction shifts modify the friction coefficient, and gravity shifts adjust gravitational acceleration. Kinematic and morphology configurations follow PAR \cite{lyu2024crossdomain}, while friction and gravity shifts follow ODRL \cite{lyu2024odrl} at a level of 0.5.

We compare our method with the following baselines: \textbf{DARC} \cite{eysenbach2021offdynamics}, which trains domain classifiers to estimate the dynamics discrepancy and modifies the reward function in the source domain; \textbf{VGDF} \cite{xu2023crossdomain}, which uses a value-guided data filtering method to select data from the source domain; \textbf{PAR} \cite{lyu2024crossdomain}, which trains encoders to estimate the representation discrepancy and modifies the reward function in the source domain; \textbf{SAC-IW}, which estimates the dynamics discrepancy as an importance sampling term for value function; \textbf{SAC-tune}, which fine-tunes the policy in the target domain for $10^5$ environmental steps; \textbf{SAC-tar} \cite{haarnoja2018soft}, which is the vanilla SAC trained in the target domain with $10^5$ environmental steps; \textbf{Oracle} \cite{haarnoja2018soft}, which is the vanilla SAC trained in the target domain with 1M environmental steps. We implement all algorithms based on the official code of ODRL \cite{lyu2024odrl} and follow the hyperparameters in the original paper. We allow all algorithms to interact with the source domain for 1M environmental steps and the target domain for $10^5$ environmental steps, \textit{i.e.}, the target domain interaction frequency $F=10$. All algorithms are trained with five random seeds.

\subsection{Adaptation Performance Evaluation}
\label{sec:adaptation-performance}

\begin{figure*} %
    \centering
    \setlength{\belowcaptionskip}{-0.4cm}
        \includegraphics[width=1.0\linewidth]{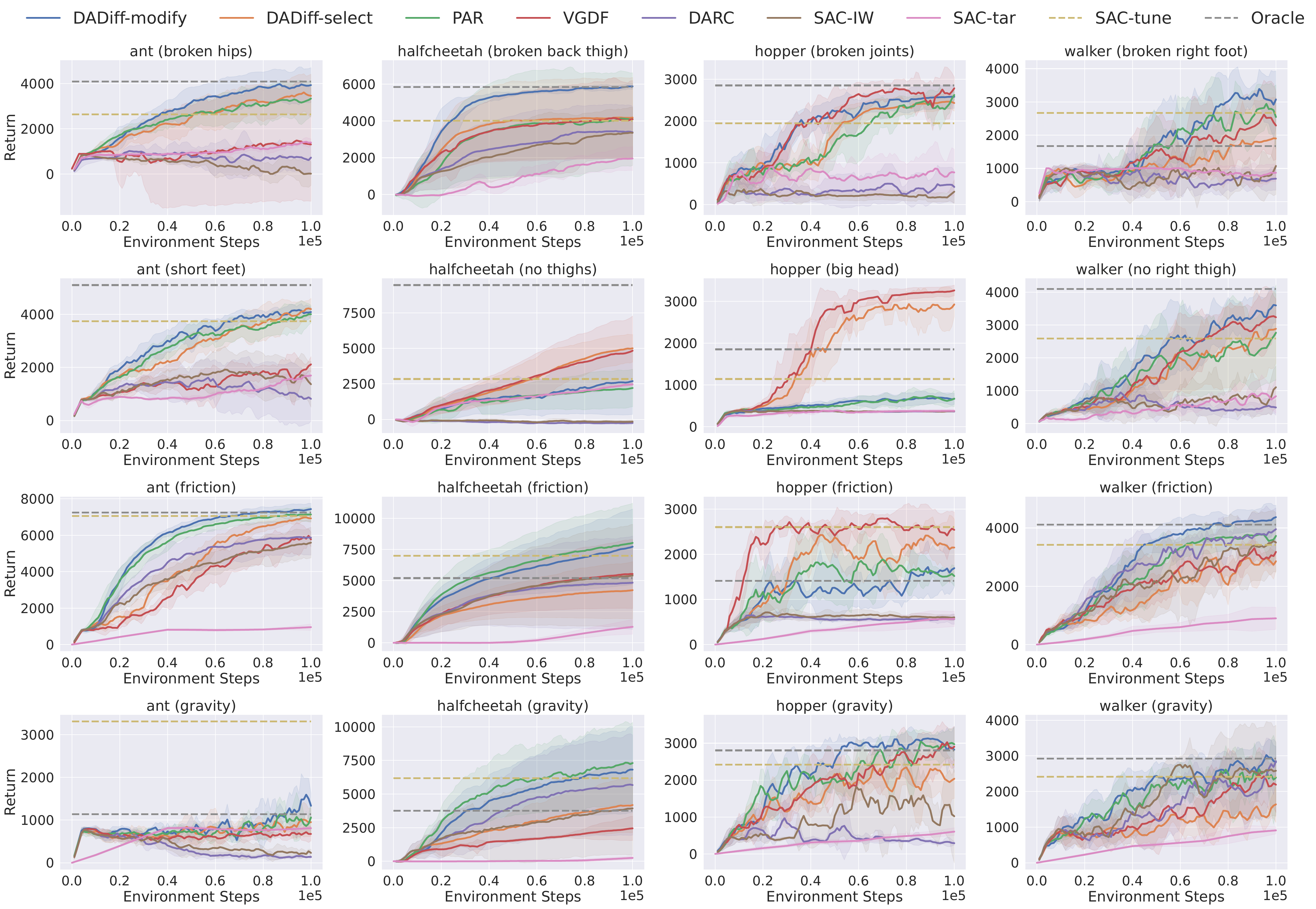}
        \caption{
            Adaptation performance under kinematic, morphology, friction, and gravity shifts (from top to bottom). The solid curves and the shaded regions denote the mean and standard deviation over five random seeds, respectively. DADiff demonstrates superior or highly competitive performance against all baselines in the majority of tasks.
        }
        \label{fig:overall-performance}
\end{figure*}

\begin{figure}
    \centering
    \setlength{\belowcaptionskip}{-0.4cm}
        \includegraphics[width=1.0\linewidth]{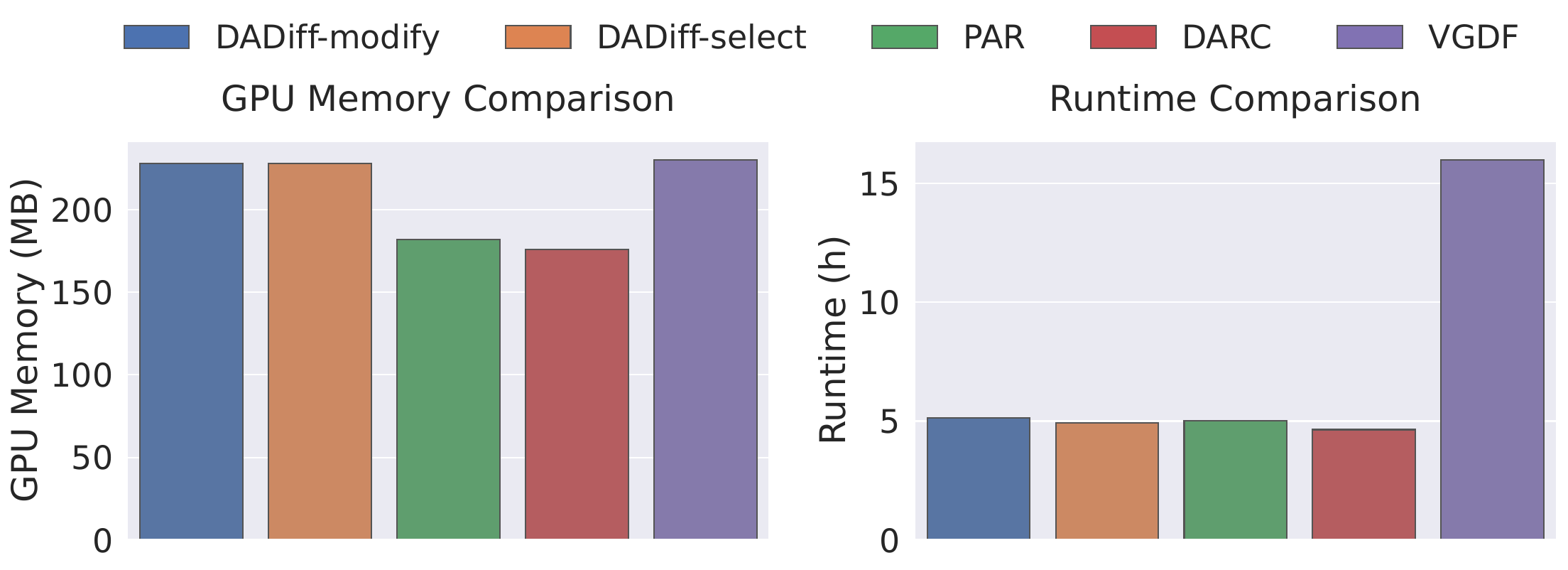}
        \caption{
            GPU memory and runtime comparisons on the \textit{halfcheetah (broken back thigh)} task. In the GPU memory comparison, \dadiff-modify and \dadiff-select exhibit slightly higher GPU memory cost compared to PAR and DARC. In the runtime comparison, VGDF requires \(3\times\) more training time than other methods due to its model-based approach.
        }
        \label{fig:runtime-comparison}
\end{figure}

We conduct experiments on sixteen tasks with diverse shifts to evaluate the adaptation performance of \dadiff and baselines. The results are summarized in Figure~\ref{fig:overall-performance}. Overall, \dadiff exhibits consistently strong performance, demonstrating superior or competitive performance against all baselines in the majority of tasks. While some existing methods, such as VGDF, PAR, or SAC-tune, occasionally reach competitive results in specific tasks, their performance fluctuates significantly across different tasks. In contrast, \dadiff maintains stable and superior adaptation performance across a wide range of shift types. We further discuss the performance of two variants of \dadiff, \dadiff-modify and \dadiff-select, respectively.

\paragraph{Reward modification variant.}

The reward modification variant of our method, \dadiff-modify, demonstrates strong and consistent performance across diverse tasks. As shown in Figure~\ref{fig:overall-performance}, it surpasses other reward modification baselines, including PAR, DARC, and SAC-IW, in most tasks and achieves performance comparable to oracle-level methods. On average, \dadiff-modify improves by 8.7\% across all sixteen tasks, with the largest gain of 42.3\% on the \textit{halfcheetah (broken back thigh)}. 
In addition to its performance advantages, we observe that our method incurs a slight increase in GPU memory usage compared to PAR and DARC due to latent state generation, as shown in Figure~\ref{fig:runtime-comparison}. This modest increase, however, contributes positively to adaptation performance by enabling better discrepancy estimation, thus representing a favorable trade-off between computational cost and effectiveness.
To further explore the performance of \dadiff-modify in stochastic environments, we provide an experiment in Section~\ref{sec:connection-dadiff-par}.

\paragraph{Data selection variant.}
In Figure~\ref{fig:overall-performance}, the data selection variant, \dadiff-select, proves to be a highly effective alternative by achieving competitive performance against top baselines in tasks where reward modification methods falter. Specifically, in the \textit{halfcheetah (no thighs)}, \textit{hopper (big head)}, and \textit{hopper (friction)} tasks, reward modification methods exhibit poor performance. In contrast, \dadiff-select achieves results that are highly competitive with the top-performing baseline, VGDF. This indicates that in certain tasks, directly filtering for transitions with low dynamics mismatch is a more effective strategy than modifying rewards. We analyze the possible reason in Section \ref{sec:reward-distribution-analysis}.
Furthermore, while VGDF demonstrates top-tier performance in these tasks, it carries significant trade-offs. Since VGDF is a model-based approach, it takes significantly longer to train by more than \(3\times\), as shown in Figure~\ref{fig:runtime-comparison}. On the other hand, \dadiff-select is able to match or exceed the performance of VGDF on such environments while maintaining comparable efficiency to similar model-free baselines.

\subsection{Parameter Study}
\label{sec:parameter-study}

The performance of \dadiff is influenced by several key hyperparameters. To better understand their roles, we conducted a series of experiments across different tasks. The results on \textit{halfcheetah (broken back thigh)} and \textit{walker (no right thigh)} are presented in Figure~\ref{fig:hyperparam-study}.

\paragraph{Penalty Coefficient $\lambda$.} $\lambda$ controls the scale of reward penalty in \dadiff-modify. As shown in Figure~\ref{fig:penalty-coefficient}, we evaluate the performance of \dadiff-modify across multiple values of $\lambda$. We find that a worse performance is often shown in the setting $\lambda=0$, where no penalty is adopted for rewards. It demonstrates the necessity of reward modification. Meanwhile, the results also indicate that the optimal value of $\lambda$ is task-dependent, and there could be multiple values that yield good performance for a specific task. For instance, in the \textit{halfcheetah (broken back thigh)} task, both $\lambda=0.5$ and $\lambda=5.0$ achieve the best performance. A poorly chosen $\lambda$ can significantly degrade performance, highlighting the importance of tuning this coefficient.

\paragraph{Data Selection Ratio $\xi\%$.} $\xi\%$ controls the ratio of source domain data to retain in \dadiff-select. As shown in Figure~\ref{fig:data-ratio}, we evaluate the performance of \dadiff-select across multiple values of $\xi\%$. Similar to the penalty coefficient, the optimal value of $\xi\%$ is task-dependent. We find that both too much ($\xi\%=100\%$) and too little (($\xi\%=0\%$)) source data can lead to suboptimal performance. As retaining too much source data may introduce transitions with significant dynamics mismatch, while retaining too little may result in insufficient data for effective learning.

\paragraph{Diffusion Timesteps $K$.} $K$ controls the number of diffusion timesteps used to measure the discrepancy in both \dadiff-modify and \dadiff-select. We provide the results of \dadiff-modify in Figure~\ref{fig:diffusion-steps-modify}. The results shows that performance improves up to $K=100$. Increasing $K$ further to 200 causes a decline, likely due to the limited capacity of the noise model, which may struggle to accurately estimate noise across too many timesteps. 

\begin{figure}
    \centering
    \setlength{\belowcaptionskip}{-0.4cm}
        \subfloat[Penalty coefficient $\lambda$.]
        {
        \includegraphics[width=1.0\linewidth]{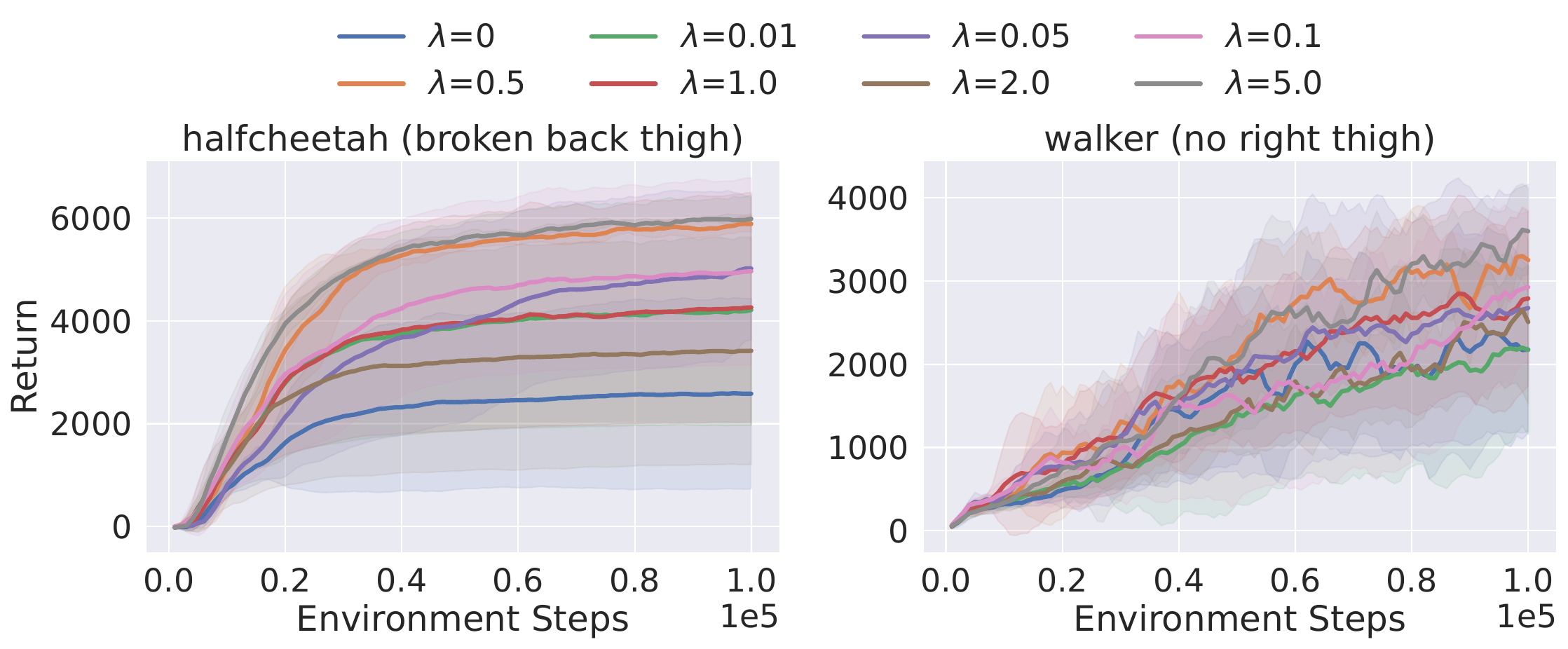}
        \label{fig:penalty-coefficient}
        }
        \\
        \subfloat[Data ratio $\xi\%$.]
        {
        \includegraphics[width=1.0\linewidth]{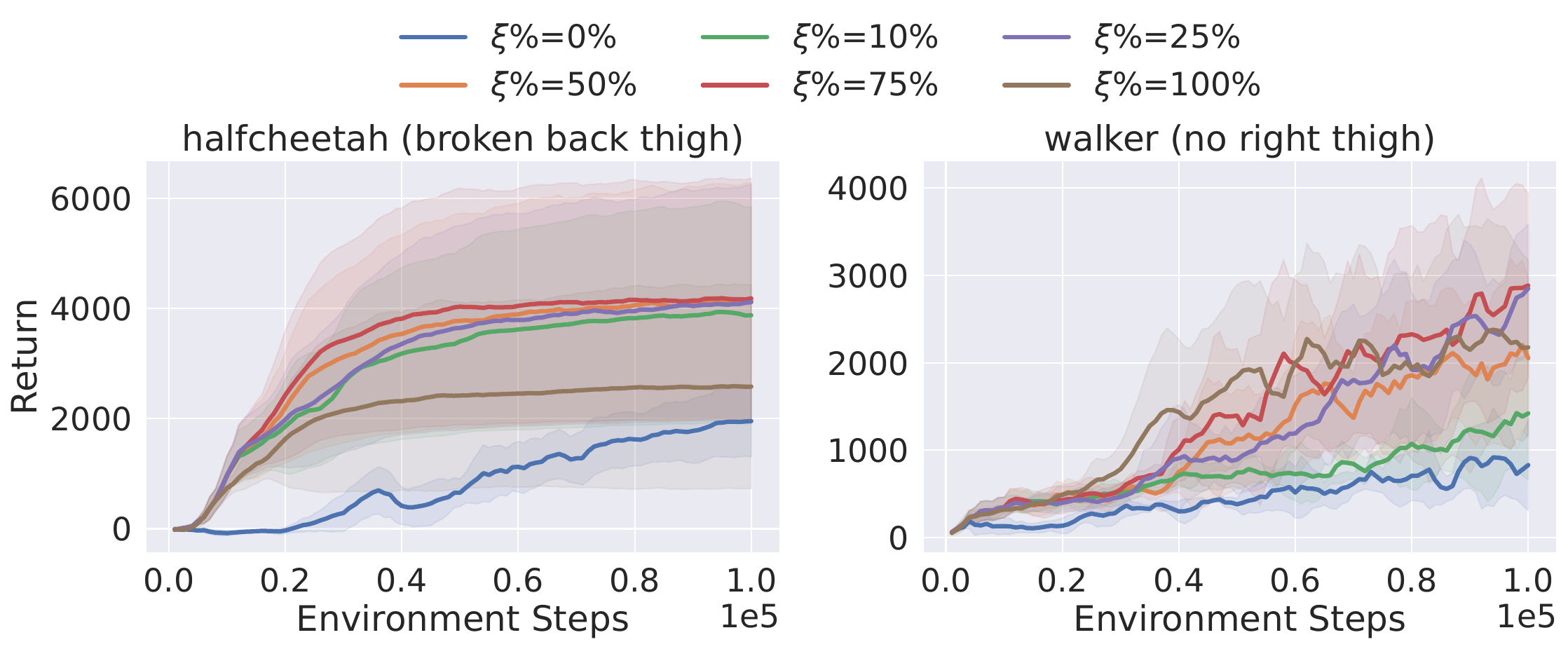}
          \label{fig:data-ratio}
        }
        \\
        \subfloat[Diffusion timesteps $K$.]
        {
        \includegraphics[width=1.0\linewidth]{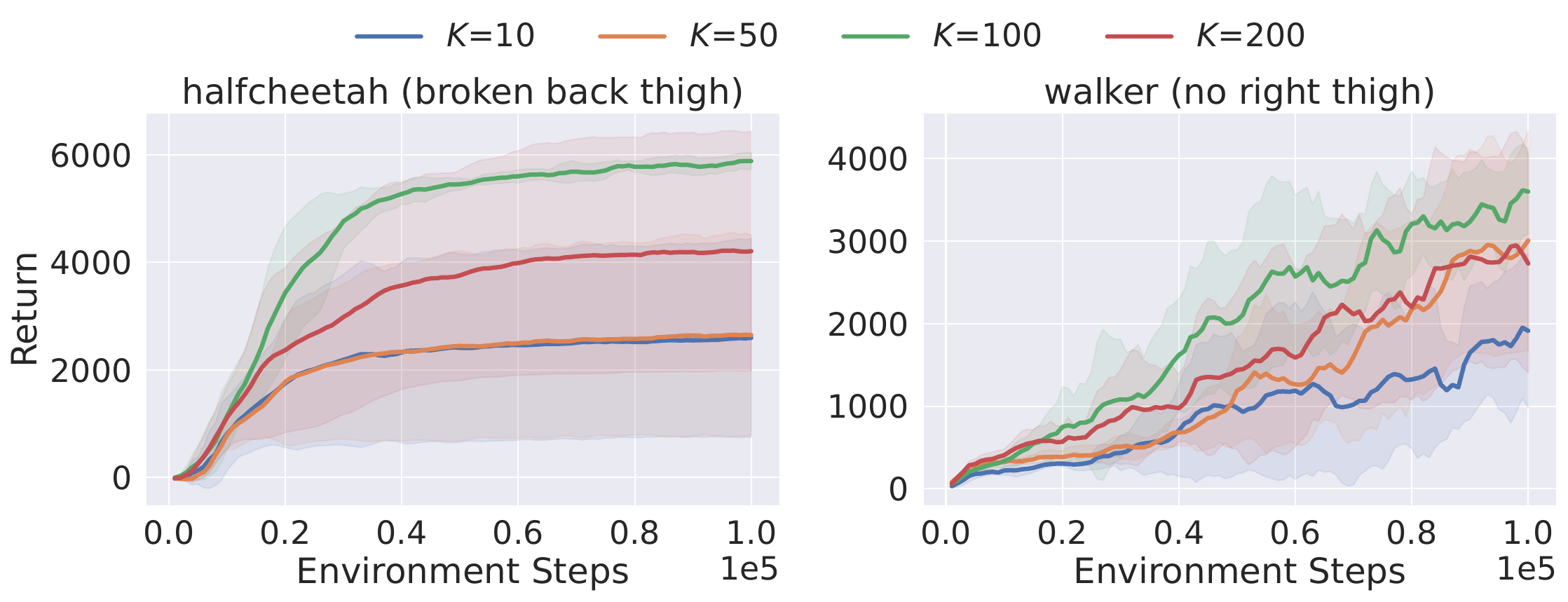}
          \label{fig:diffusion-steps-modify}
        }
    \caption{
        Parameter study. The solid curves and the shaded regions denote the mean and standard deviation over five random seeds, respectively.
    }
    \label{fig:hyperparam-study}
\end{figure}

\section{Discussions}
\label{sec:discussion}

\subsection{Connection between \dadiff and PAR}
\label{sec:connection-dadiff-par}

We explore the connection between PAR and our method from a theoretical perspective. The performance bound of our method is controlled by the generative trajectory discrepancy in Theorem~\ref{theorem:performance-bound-latent-state-discrepancy}. We consider a special case, where the number of latent states in the trajectory is $K=1$.
Instead of considering latent states in the generative trajectory, we take $s'_1$ as a latent representation and introduce the one-to-one representation mapping assumption in PAR \cite{lyu2024crossdomain}, which assumes that there exists a one-to-one mapping for each state-action pair $(s, a)$ and its latent representation $s'_1$. In this setting, the state-action pair $(s, a)$ in Equation~\ref{eq:performance-bound-latent-state-discrepancy} can be all replaced by the corresponding latent representation $s'_1$.
Therefore, the performance bound can be rewritten as follows,
\begin{equation}
    \resizebox{\linewidth}{!}{$
    \begin{aligned}
    \eta_{\mathcal{M}_\mathrm{src}}(\pi) & - \eta_{\mathcal{M}_\mathrm{tar}}(\pi) \le \\
    & \frac{\sqrt{2}\gamma r_{\mathrm{max}}}{(1 - \gamma)^2} \mathbb{E}_{\rho_\mathrm{src}^\pi} \left[ \sqrt { \mathbb{E}_{P_\mathrm{src}} \left[ D_\mathrm{KL}(P_\mathrm{src}(s'_0 | s'_1) || P_\mathrm{tar}(s'_0 | s'_1)) \right] } \right].
    \end{aligned}
    $}
\end{equation}
We further introduce a conclusion proven in PAR \cite{lyu2024crossdomain}: %
\begin{equation}
    \resizebox{\linewidth}{!}{$
    \begin{aligned}
    D_\mathrm{KL}(P_\mathrm{src}(s'_1 | s'_0) || & P_\mathrm{tar}(s'_1 | s'_0)) = \\
    & D_\mathrm{KL}(P_\mathrm{src}(s'_0 | s'_1) || P_\mathrm{tar}(s'_0 | s'_1)) + \mathbb{H}(s'_\mathrm{src}) - \mathbb{H}(s'_\mathrm{tar}).
    \end{aligned}
    $}
\end{equation}
Therefore, the performance bound can be rewritten as follows,
\begin{equation}
    \resizebox{\linewidth}{!}{$
    \begin{aligned}
    \eta_{\mathcal{M}_\mathrm{src}}(\pi) - & \eta_{\mathcal{M}_\mathrm{tar}}(\pi) \le \\
    & \frac{\sqrt{2}\gamma r_{\mathrm{max}}}{(1 - \gamma)^2} \mathbb{E}_{\rho_\mathrm{src}^\pi} \left[ \sqrt { \mathbb{E}_{P_\mathrm{src}} \left[ D_\mathrm{KL}(P_\mathrm{src}(s'_1 | s'_0) || P_\mathrm{tar}(s'_1 | s'_0)) \right] } \right] +\\
    & \frac{\sqrt{2}\gamma r_{\mathrm{max}}}{(1 - \gamma)^2} \mathbb{E}_{\rho_\mathrm{src}^\pi} \left[ \sqrt { \mathbb{E}_{P_\mathrm{src}} \left[ \mathbb{H}(s'_\mathrm{src}) - \mathbb{H}(s'_\mathrm{tar}) \right] } \right].
    \end{aligned}
    $}
\end{equation}
This performance bound is consistent with the performance bound of PAR, which indicates that PAR can be considered as a special case of our method. However, the one-to-one representation mapping assumption may not hold in practice, especially in stochastic environments, which limits the application of PAR. In contrast, our method does not rely on this assumption and can handle more general scenarios. We validate this point in environments with stochastic dynamics. Noises with different standard deviation $\varsigma$ are introduced to the actions to simulate stochastic dynamics, and two tasks with kinematic shifts, \textit{hopper (broken joints)} and \textit{walker (broken right foot)}, are considered. We evaluate the performance of \dadiff-modify and PAR, which is presented in Table~\ref{tab:stochastic-performance}. Notably, our method maintains robust performance even as the standard deviation $\varsigma$ increases, while PAR's performance degrades significantly. We believe the decrease in PAR's performance is due to its reliance on one-to-one representation assumptions, which may not hold in stochastic settings.

\begin{table}
\centering
\vspace{0.15cm}
\caption{
Adaptation performance under stochastic dynamics controlled by the standard deviation parameter \(\varsigma\). 
Average return and standard deviation over five random seeds are reported. 
The best results are in \textbf{bold}, and performance change relative to the deterministic setting (\(\varsigma = 0.0\)) is shown in parentheses.
}
\label{tab:stochastic-performance}
\resizebox{\linewidth}{!}{
\begin{tabular}{cccc}
\toprule
Environment & $\varsigma$ & \textbf{\dadiff-modify} & \textbf{PAR} \\
\midrule

\multirow{4}{*}{hopper (broken joints)} 
 & 0.00 & 2582.1$\pm$251.6 & \textbf{2623.1$\pm$105.2} \\
 & 0.01 & \textbf{2591.0$\pm$159.2} ($\uparrow$0.34\%) & 2398.3$\pm$297.8 ($\downarrow$8.57\%) \\
 & 0.02 & \textbf{2515.9$\pm$101.8} ($\downarrow$2.57\%) & 2328.7$\pm$302.9 ($\downarrow$11.22\%) \\
 & 0.03 & \textbf{2574.2$\pm$280.6} ($\downarrow$0.31\%) & 2406.1$\pm$455.7 ($\downarrow$8.27\%) \\

\midrule

\multirow{4}{*}{walker (broken right foot)} 
 & 0.00 & \textbf{3390.4$\pm$464.4} & 2943.3$\pm$546.7 \\
 & 0.01 & \textbf{2879.3$\pm$688.9} ($\downarrow$15.08\%) & 2373.8$\pm$1072.4 ($\downarrow$19.35\%) \\
 & 0.02 & 2812.5$\pm$934.6 ($\downarrow$17.05\%) & \textbf{2825.8$\pm$466.6} ($\downarrow$3.99\%) \\
 & 0.03 & \textbf{3176.8$\pm$796.4} ($\downarrow$6.30\%) & 1613.9$\pm$878.7 ($\downarrow$45.17\%) \\

\bottomrule
\end{tabular}
}
\end{table}

\subsection{Reward Distribution Analysis}
\label{sec:reward-distribution-analysis}

\begin{figure}
    \centering
    \setlength{\belowcaptionskip}{-0.5cm}
        \includegraphics[width=1.0\linewidth]{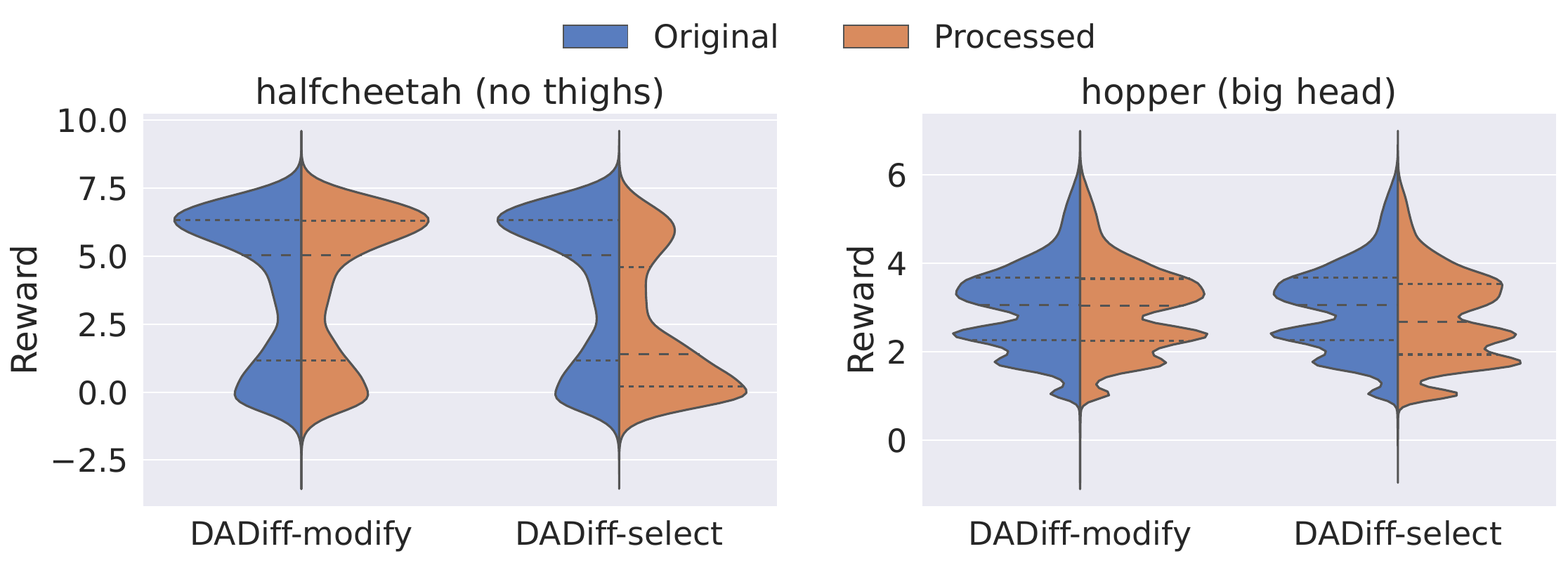}
        \caption{
            Reward distribution comparison between the source-domain rewards before processing (Original) and after modification or selection (Processed).
        }
        \label{fig:reward-distribution}
\end{figure}

We further examine the reasons behind the superior performance of \dadiff-select, in contrast to the severe failure of \dadiff-modify on \textit{halfcheetah (no thighs)} and \textit{hopper (big head)} tasks, as illustrated in Figure~\ref{fig:overall-performance}. Specifically, we analyze the reward distributions of source-domain data after modification or selection. The results are presented in Figure~\ref{fig:reward-distribution}. We find that \dadiff-select generates a higher distribution in the low-reward region compared to \dadiff-modify on both tasks. This suggests that the low-reward data may play a crucial role in these tasks, which can effectively guide the policy to avoid undesirable states and actions.

\section{Conclusion}
\label{sec:conclusion}

This work explores the problem of online dynamics adaptation in reinforcement learning from a generative modeling perspective. We first theoretically analyze the performance bound of a policy in the source and target domains, which is controlled by the generative trajectory discrepancy. Based on this analysis, we propose a novel method, \dadiff, which utilizes diffusion models to measure the dynamics discrepancy and performs either reward modification or data selection to adapt to the target domain. Extensive experiments demonstrate that our method outperforms existing baselines in tasks with various shifts. %

\section*{APPENDIX}

\subsection{Useful Lemmas}
\label{sec:useful-lemmas}

\begin{lemma}{
    (Telescoping lemma.)
    Denote $\mathcal{M}_1 = (\mathcal{S}, \mathcal{A}, P_1, r, \gamma)$ and $\mathcal{M}_2 = (\mathcal{S}, \mathcal{A}, P_2, r, \gamma)$ as two MDPs with the same state and action spaces but different transition dynamics $P_1$ and $P_2$. The performance difference of a policy $\pi$ evaluated in $\mathcal{M}_1$ and $\mathcal{M}_2$ can be expressed as:
    \begin{equation*}
        \resizebox{\linewidth}{!}{$
        \eta_{\mathcal{M}_1}(\pi) - \eta_{\mathcal{M}_2}(\pi) = \frac{\gamma}{1-\gamma} \mathbb{E}_{\rho^\pi_{\mathcal{M}_1}(s,a)} \left[\mathbb{E}_{s' \sim P_1}[V_{\mathcal{M}_2}^\pi(s')] - \mathbb{E}_{s' \sim P_2}[V_{\mathcal{M}_2}^\pi(s')]\right]
        $}
        \label{eq:telescoping}
    \end{equation*}

    \label{lemma:telescoping}
}
\end{lemma}
\vspace{-0.1cm}
\textbf{\textit{Proof.}} Please see Lemma 4.3 in SLBO \cite{luo2018algorithmic} for a detailed proof.

\subsection{Proof of Theorem~\ref{theorem:performance-bound-latent-state-discrepancy}}
\label{sec:proof-theorem}

\begin{theorem}{
    (Performance bound controlled by generative trajectory discrepancy.)
    Denote $\mathcal{M}_\mathrm{src}$ and $\mathcal{M}_\mathrm{tar}$ as the source and target domains with different dynamics, respectively. The performance difference of any policy $\pi$ evaluated in $\mathcal{M}_\mathrm{src}$ and $\mathcal{M}_\mathrm{tar}$ can be bounded as below,
    \begin{equation*}
        \resizebox{\linewidth}{!}{$
            \begin{aligned}
            & \eta_{\mathcal{M}_\mathrm{src}}(\pi) - \eta_{\mathcal{M}_\mathrm{tar}}(\pi)
            \le \\
            & \frac{\sqrt{2}\gamma r_{\mathrm{max}}}{(1 - \gamma)^2}
            \underbrace{\mathbb{E}_{\rho_\mathrm{src}^\pi}\!\left[
            \sqrt{\mathbb{E}_{P_\mathrm{src}}\!\left[
            D_\mathrm{KL}\!\left(P_\mathrm{src}(s'_K \mid s, a)\,\|\,P_\mathrm{tar}(s'_K \mid s, a)\right)
            \right]}
            \right]}_{(a):\ \mathrm{initial\ latent\ state\ deviation}} + \\
            & \frac{\sqrt{2}\gamma r_{\mathrm{max}}}{(1 - \gamma)^2}
            \underbrace{\mathbb{E}_{\rho_\mathrm{src}^\pi}\!\left[
            \sqrt{\mathbb{E}_{P_\mathrm{src}}\!\left[
            \sum_{k=1}^{K}
            D_\mathrm{KL}\!\left(P_\mathrm{src}(s'_{k-1} \mid s'_k, s, a)\,\|\,P_\mathrm{tar}(s'_{k-1} \mid s'_k, s, a)\right)
            \right]}
            \right]}_{(b):\ \mathrm{latent\ state\ transition\ mismatch}}.
            \end{aligned}
        $}
    \end{equation*}
}
\end{theorem}
\vspace{0.2cm}

\textbf{\textit{Proof.}} As the value function $V^\pi_\mathcal{M}(s)$ estimates the expected return of a policy $\pi$ starting from state $s$ in domain $\mathcal{M}$, and the rewards are bounded, we have $|V^\pi_\mathcal{M}(s)| \le r_{\mathrm{max}}/(1 - \gamma), \forall s$. By using Lemma~\ref{lemma:telescoping}, we have:
\vspace{-0.4cm}
\begin{center}
\resizebox{\columnwidth}{!}{%
\begin{minipage}{1.02\columnwidth}
{\footnotesize
\begin{subequations}
\begin{align}
    \eta_{\mathcal{M}_\mathrm{src}}&(\pi) - \eta_{\mathcal{M}_\mathrm{tar}}(\pi)
    = \frac{\gamma}{1-\gamma} \mathbb{E}_{\rho_\mathrm{src}^\pi} \!\left[\mathbb{E}_{P_\mathrm{src}}[r(s,a)] - \mathbb{E}_{P_\mathrm{tar}}[r(s,a)]\right] \notag\\
    & = \frac{\gamma}{1 - \gamma}  \mathbb{E}_{\rho_\mathrm{src}^\pi} \!\left[ \int_{s'_0} P_\mathrm{src}(s'_0|s,a) V_\mathrm{tar}^{\pi}(s'_0) - \int_{s'_0} P_\mathrm{tar}(s'_0|s,a) V_\mathrm{tar}^{\pi}(s'_0) ds'_0 \right] \notag\\
    & \le \frac{\gamma}{1 - \gamma}  \mathbb{E}_{\rho_\mathrm{src}^\pi} \!\left[ \int_{s'_0} (P_\mathrm{src}(s'_0|s,a) - P_\mathrm{tar}(s'_0|s,a)) \left| V_\mathrm{tar}^{\pi}(s'_0) \right| ds'_0 \right] \notag\\
    & \le \frac{\gamma r_{\mathrm{max}}}{(1 - \gamma)^2}  \mathbb{E}_{\rho_\mathrm{src}^\pi} \!\left[ \int_{s'_0} P_\mathrm{src}(s'_0|s,a) - P_\mathrm{tar}(s'_0|s,a) ds'_0 \right] \notag\\
    & = \frac{\gamma r_{\mathrm{max}}}{(1 - \gamma)^2}  \mathbb{E}_{\rho_\mathrm{src}^\pi} \!\left[ \int_{s'_{0:K}} P_\mathrm{src}(s'_{0:K}|s,a) - P_\mathrm{tar}(s'_{0:K}|s,a) ds'_{0:K} \right] \notag\\
    & = \frac{2 \gamma r_{\mathrm{max}}}{(1 - \gamma)^2}  \mathbb{E}_{\rho_\mathrm{src}^\pi} \!\left[ D_\mathrm{TV} (P_\mathrm{src}(s'_{0:K}|s,a) || P_\mathrm{tar}(s'_{0:K}|s,a)) \right] \notag\\
    & \le \frac{\sqrt{2}\gamma r_{\mathrm{max}}}{(1 - \gamma)^2} \mathbb{E}_{\rho_\mathrm{src}^\pi} \!\left[ \sqrt{ D_\mathrm{KL}(P_\mathrm{src}(s'_{0:K}|s,a) \Vert P_\mathrm{tar}(s'_{0:K}|s,a)) } \right] \tag{a}\\
    & = \frac{\sqrt{2}\gamma r_{\mathrm{max}}}{(1 - \gamma)^2} \mathbb{E}_{\rho_\mathrm{src}^\pi} \!\left[ \sqrt{ \mathbb{E} _{P_\mathrm{src}} \!\left[ \log \frac{P_\mathrm{src}(s'_{0:K}|s,a)}{P_\mathrm{tar}(s'_{0:K}|s,a)} \right] } \right] \notag\\
    & = \frac{\sqrt{2}\gamma r_{\mathrm{max}}}{(1 - \gamma)^2} \mathbb{E}_{\rho_\mathrm{src}^\pi} \!\left[ \sqrt{ \mathbb{E} _{P_\mathrm{src}} \!\left[ \log \frac{P_\mathrm{src}(s'_K | s, a)}{P_\mathrm{tar}(s'_K | s, a)} + \sum_{k=1}^{K} \log \frac{P_\mathrm{src} (s'_{k-1} | s'_k, s, a)}{P_\mathrm{tar}(s'_{k-1} | s'_k, s, a)} \right] } \right] \tag{b}\\
    & \le \frac{\sqrt{2}\gamma r_{\mathrm{max}}}{(1 - \gamma)^2} \mathbb{E}_{\rho_\mathrm{src}^\pi} \!\left[ \sqrt { \mathbb{E}_{P_\mathrm{src}} \!\left[ D_\mathrm{KL}(P_\mathrm{src}(s'_K | s, a) || P_\mathrm{tar}(s'_K | s, a)) \right] } \right] + \notag\\
    & \quad \frac{\sqrt{2}\gamma r_{\mathrm{max}}}{(1 - \gamma)^2} \mathbb{E}_{\rho_\mathrm{src}^\pi} \!\left[ \sqrt { \mathbb{E}_{P_\mathrm{src}} \!\left[ \sum_{k=1}^{K} D_\mathrm{KL}(P_\mathrm{src}(s'_{k-1} | s'_k, s, a) || P_\mathrm{tar}(s'_{k-1} | s'_k, s, a)) \right] } \right] \tag{c}
\end{align}
\end{subequations}
}
\end{minipage}%
}
\end{center}
where $D_\mathrm{TV} (P || Q)$ is the total variation distance between two distributions $P$ and $Q$, the step (a) holds by Pinsker's inequality, the step (b) holds by the Markov property, and the step (c) holds by the subadditivity of the square root function. The proof shows that the performance difference can be controlled by the distributional divergence of latent states in generative trajectories.

\section*{ACKNOWLEDGMENT}

The work was partially supported by NSF award \#2442477, \#2550203 and \#2536297. The views and conclusions in this paper should not be interpreted as representing any funding agencies.

\bibliography{IEEEexample}

@article{xu2023crossdomain,
  title={Cross-domain policy adaptation via value-guided data filtering},
  author={Xu, Kang and Bai, Chenjia and Ma, Xiaoteng and Wang, Dong and Zhao, Bin and Wang, Zhen and Li, Xuelong and Li, Wei},
  journal={Advances in Neural Information Processing Systems},
  volume={36},
  pages={73395--73421},
  year={2023}
}

@article{eysenbach2021offdynamics,
  title={Off-dynamics reinforcement learning: Training for transfer with domain classifiers},
  author={Eysenbach, Benjamin and Asawa, Swapnil and Chaudhari, Shreyas and Levine, Sergey and Salakhutdinov, Ruslan},
  journal={arXiv preprint arXiv:2006.13916},
  year={2020}
}

@inproceedings{lyu2024crossdomain,
  title={Cross-domain policy adaptation by capturing representation mismatch},
  author={Lyu, Jiafei and Bai, Chenjia and Yang, Jingwen and Lu, Zongqing and Li, Xiu},
  booktitle={Proceedings of the 41st International Conference on Machine Learning},
  pages={33638--33663},
  year={2024}
}

@article{ho2020denoising,
  title={Denoising diffusion probabilistic models},
  author={Ho, Jonathan and Jain, Ajay and Abbeel, Pieter},
  journal={Advances in neural information processing systems},
  volume={33},
  pages={6840--6851},
  year={2020}
}

@inproceedings{sohl2015deep,
  title={Deep unsupervised learning using nonequilibrium thermodynamics},
  author={Sohl-Dickstein, Jascha and Weiss, Eric and Maheswaranathan, Niru and Ganguli, Surya},
  booktitle={International conference on machine learning},
  pages={2256--2265},
  year={2015},
  organization={pmlr}
}

@inproceedings{zhao2020sim,
  title={Sim-to-Real Transfer in Deep Reinforcement Learning for Robotics: a Survey},
  author={Zhao, Wenshuai and Queralta, Jorge Pe{\~n}a and Westerlund, Tomi},
  booktitle={2020 IEEE symposium series on computational intelligence (SSCI)},
  pages={737--744},
  year={2020},
  organization={IEEE}
}

@article{da2025survey,
  title={A Survey of Sim-to-Real Methods in RL: Progress, Prospects and Challenges with Foundation Models},
  author={Da, Longchao and Turnau, Justin and Kutralingam, Thirulogasankar Pranav and Velasquez, Alvaro and Shakarian, Paulo and Wei, Hua},
  journal={arXiv preprint arXiv:2502.13187},
  year={2025}
}

@inproceedings{peng2018sim,
  title={Sim-to-real transfer of robotic control with dynamics randomization},
  author={Peng, Xue Bin and Andrychowicz, Marcin and Zaremba, Wojciech and Abbeel, Pieter},
  booktitle={2018 IEEE international conference on robotics and automation (ICRA)},
  pages={3803--3810},
  year={2018},
  organization={IEEE}
}

@inproceedings{mehta2020active,
  title={Active Domain Randomization},
  author={Mehta, Bhairav and Diaz, Manfred and Golemo, Florian and Pal, Christopher J and Paull, Liam},
  booktitle={Conference on Robot Learning},
  pages={1162--1176},
  year={2020},
  organization={PMLR}
}

@inproceedings{curtis2025flowbased,
  title={Flow-based Domain Randomization for Learning and Sequencing Robotic Skills},
  author={Curtis, Aidan and Li, Eric and Noseworthy, Michael and Gothoskar, Nishad and Chitta, Sachin and Li, Hui and Kaelbling, Leslie Pack and Carey, Nicole E},
  booktitle={Forty-second International Conference on Machine Learning},
  year={2025}
}

@inproceedings{chebotar2019closing,
  title={Closing the sim-to-real loop: Adapting simulation randomization with real world experience},
  author={Chebotar, Yevgen and Handa, Ankur and Makoviychuk, Viktor and Macklin, Miles and Issac, Jan and Ratliff, Nathan and Fox, Dieter},
  booktitle={2019 International Conference on Robotics and Automation (ICRA)},
  pages={8973--8979},
  year={2019},
  organization={IEEE}
}

@inproceedings{todorov2012mujoco,
  title={Mujoco: A physics engine for model-based control},
  author={Todorov, Emanuel and Erez, Tom and Tassa, Yuval},
  booktitle={2012 IEEE/RSJ international conference on intelligent robots and systems},
  pages={5026--5033},
  year={2012},
  organization={IEEE}
}

@article{brockman2016openai,
  title={Openai gym},
  author={Brockman, Greg and Cheung, Vicki and Pettersson, Ludwig and Schneider, Jonas and Schulman, John and Tang, Jie and Zaremba, Wojciech},
  journal={arXiv preprint arXiv:1606.01540},
  year={2016}
}

@article{haarnoja2018soft,
  title={Soft actor-critic algorithms and applications},
  author={Haarnoja, Tuomas and Zhou, Aurick and Hartikainen, Kristian and Tucker, George and Ha, Sehoon and Tan, Jie and Kumar, Vikash and Zhu, Henry and Gupta, Abhishek and Abbeel, Pieter and others},
  journal={arXiv preprint arXiv:1812.05905},
  year={2018}
}

@article{lyu2024odrl,
  title={Odrl: A benchmark for off-dynamics reinforcement learning},
  author={Lyu, Jiafei and Xu, Kang and Xu, Jiacheng and Yang, Jing-Wen and Zhang, Zongzhang and Bai, Chenjia and Lu, Zongqing and Li, Xiu and others},
  journal={Advances in Neural Information Processing Systems},
  volume={37},
  pages={59859--59911},
  year={2024}
}

@article{luo2018algorithmic,
  title={Algorithmic framework for model-based deep reinforcement learning with theoretical guarantees},
  author={Luo, Yuping and Xu, Huazhe and Li, Yuanzhi and Tian, Yuandong and Darrell, Trevor and Ma, Tengyu},
  journal={arXiv preprint arXiv:1807.03858},
  year={2018}
}

@article{xue2023state,
  title={State regularized policy optimization on data with dynamics shift},
  author={Xue, Zhenghai and Cai, Qingpeng and Liu, Shuchang and Zheng, Dong and Jiang, Peng and Gai, Kun and An, Bo},
  journal={Advances in neural information processing systems},
  volume={36},
  pages={32926--32937},
  year={2023}
}

@inproceedings{ge2023policy,
  title={Policy adaptation from foundation model feedback},
  author={Ge, Yuying and Macaluso, Annabella and Li, Li Erran and Luo, Ping and Wang, Xiaolong},
  booktitle={Proceedings of the IEEE/CVF Conference on Computer Vision and Pattern Recognition},
  pages={19059--19069},
  year={2023}
}

@article{pan2025crossdomain,
  title={Cross-Domain Reinforcement Learning Under Distinct State-Action Spaces Via Hybrid Q Functions},
  author={Pan, Kuan-Chen and Chen, MingHong and Huang, You-De and Liu, Xi and Hsieh, Ping-Chun}
}

@article{slaoui2019robust,
  title={Robust visual domain randomization for reinforcement learning},
  author={Slaoui, Reda Bahi and Clements, William R and Foerster, Jakob N and Toth, S{\'e}bastien},
  journal={arXiv preprint arXiv:1910.10537},
  year={2019}
}

@article{jiang2023variance,
  title={Variance reduced domain randomization for reinforcement learning with policy gradient},
  author={Jiang, Yuankun and Li, Chenglin and Dai, Wenrui and Zou, Junni and Xiong, Hongkai},
  journal={IEEE Transactions on Pattern Analysis and Machine Intelligence},
  volume={46},
  number={2},
  pages={1031--1048},
  year={2023},
  publisher={IEEE}
}

@article{nagabandi2018learning,
  title={Learning to adapt in dynamic, real-world environments through meta-reinforcement learning},
  author={Nagabandi, Anusha and Clavera, Ignasi and Liu, Simin and Fearing, Ronald S and Abbeel, Pieter and Levine, Sergey and Finn, Chelsea},
  journal={arXiv preprint arXiv:1803.11347},
  year={2018}
}

@article{wu2022zero,
  title={Zero-shot policy transfer with disentangled task representation of meta-reinforcement learning},
  author={Wu, Zheng and Xie, Yichen and Lian, Wenzhao and Wang, Changhao and Guo, Yanjiang and Chen, Jianyu and Schaal, Stefan and Tomizuka, Masayoshi},
  journal={arXiv preprint arXiv:2210.00350},
  year={2022}
}

@inproceedings{raychaudhuri2021cross,
  title={Cross-domain imitation from observations},
  author={Raychaudhuri, Dripta S and Paul, Sujoy and Vanbaar, Jeroen and Roy-Chowdhury, Amit K},
  booktitle={International conference on machine learning},
  pages={8902--8912},
  year={2021},
  organization={PMLR}
}

@article{fickinger2022crossdomain,
  title={Cross-domain imitation learning via optimal transport},
  author={Fickinger, Arnaud and Cohen, Samuel and Russell, Stuart and Amos, Brandon},
  journal={arXiv preprint arXiv:2110.03684},
  year={2021}
}

@article{wen2024contrastive,
  title={Contrastive representation for data filtering in cross-domain offline reinforcement learning},
  author={Wen, Xiaoyu and Bai, Chenjia and Xu, Kang and Yu, Xudong and Zhang, Yang and Li, Xuelong and Wang, Zhen},
  journal={arXiv preprint arXiv:2405.06192},
  year={2024}
}

@article{guo2024off,
  title={Off-dynamics reinforcement learning via domain adaptation and reward augmented imitation},
  author={Guo, Yihong and Wang, Yixuan and Shi, Yuanyuan and Xu, Pan and Liu, Anqi},
  journal={Advances in Neural Information Processing Systems},
  volume={37},
  pages={136326--136360},
  year={2024}
}

@article{van2024policy,
  title={Policy Learning for Off-Dynamics RL with Deficient Support},
  author={Van, Linh Le Pham and Tran, Hung The and Gupta, Sunil},
  journal={arXiv preprint arXiv:2402.10765},
  year={2024}
}

@article{lyu2024cross,
  title={Cross-domain policy adaptation by capturing representation mismatch},
  author={Lyu, Jiafei and Bai, Chenjia and Yang, Jingwen and Lu, Zongqing and Li, Xiu},
  journal={arXiv preprint arXiv:2405.15369},
  year={2024}
}

@article{chi2025diffusion,
  title={Diffusion policy: Visuomotor policy learning via action diffusion},
  author={Chi, Cheng and Xu, Zhenjia and Feng, Siyuan and Cousineau, Eric and Du, Yilun and Burchfiel, Benjamin and Tedrake, Russ and Song, Shuran},
  journal={The International Journal of Robotics Research},
  volume={44},
  number={10-11},
  pages={1684--1704},
  year={2025},
  publisher={Sage Publications Sage UK: London, England}
}

@article{kang2023efficient,
  title={Efficient diffusion policies for offline reinforcement learning},
  author={Kang, Bingyi and Ma, Xiao and Du, Chao and Pang, Tianyu and Yan, Shuicheng},
  journal={Advances in Neural Information Processing Systems},
  volume={36},
  pages={67195--67212},
  year={2023}
}

@article{he2023diffusion,
  title={Diffusion model is an effective planner and data synthesizer for multi-task reinforcement learning},
  author={He, Haoran and Bai, Chenjia and Xu, Kang and Yang, Zhuoran and Zhang, Weinan and Wang, Dong and Zhao, Bin and Li, Xuelong},
  journal={Advances in neural information processing systems},
  volume={36},
  pages={64896--64917},
  year={2023}
}

@article{lu2023synthetic,
  title={Synthetic experience replay},
  author={Lu, Cong and Ball, Philip and Teh, Yee Whye and Parker-Holder, Jack},
  journal={Advances in Neural Information Processing Systems},
  volume={36},
  pages={46323--46344},
  year={2023}
}

@article{wang2024diffusion,
  title={Diffusion actor-critic with entropy regulator},
  author={Wang, Yinuo and Wang, Likun and Jiang, Yuxuan and Zou, Wenjun and Liu, Tong and Song, Xujie and Wang, Wenxuan and Xiao, Liming and Wu, Jiang and Duan, Jingliang and others},
  journal={Advances in Neural Information Processing Systems},
  volume={37},
  pages={54183--54204},
  year={2024}
}

@article{zhu2024madiff,
  title={Madiff: Offline multi-agent learning with diffusion models},
  author={Zhu, Zhengbang and Liu, Minghuan and Mao, Liyuan and Kang, Bingyi and Xu, Minkai and Yu, Yong and Ermon, Stefano and Zhang, Weinan},
  journal={Advances in Neural Information Processing Systems},
  volume={37},
  pages={4177--4206},
  year={2024}
}

@article{van2025dmc,
  title={DmC: Nearest Neighbor Guidance Diffusion Model for Offline Cross-domain Reinforcement Learning},
  author={Van, Linh Le Pham and Nguyen, Minh Hoang and Kieu, Duc and Le, Hung and Tran, Hung The and Gupta, Sunil},
  journal={arXiv preprint arXiv:2507.20499},
  year={2025}
}
\bibliographystyle{IEEEtran}

\end{document}